\documentclass[conference]{IEEEtran}

\makeatletter
\def\ps@IEEEtitlepagestyle{%
  \def\@oddfoot{\mycopyrightnotice}%
  \def\@evenfoot{}%
}
\def\mycopyrightnotice{%
  {\footnotesize 979-8-3503-9118-3/24/\$31.00 ~\copyright~2024 IEEE \hfill}
  \gdef\mycopyrightnotice{}
}
\usepackage{comment}
\usepackage{caption}
\usepackage{tikz}
\usetikzlibrary{positioning}
\usepackage{blindtext}
\usepackage{eso-pic}
\IEEEoverridecommandlockouts
\usepackage{amsmath,amssymb,amsfonts}
\usepackage{amsthm}
\newtheorem{theorem}{Theorem}[section]

\newtheorem{proposition}[theorem]{Proposition}
\newtheorem{definition}[theorem]{Definition}
\newtheorem{remark}[theorem]{Remark}
\usepackage{algorithmic}
\usepackage{graphicx}
\usepackage{textcomp}
\usepackage{xcolor}
\def\BibTeX{{\rm B\kern-.05em{\sc i\kern-.025em b}\kern-.08em
    T\kern-.1667em\lower.7ex\hbox{E}\kern-.125emX}}
\DeclareCaptionFont{myfont}{\fontsize{8pt}{11pt}\selectfont}
\usepackage{eso-pic}
\newcommand\AtPageUpperMyright[1]{\AtPageUpperLeft{%
 \put(\LenToUnit{0.17\paperwidth},\LenToUnit{-2cm}){%
     \parbox{0.9\textwidth}{\raggedleft\fontsize{8}{11}\selectfont #1}}%
 }}%
\newcommand{\conf}[1]{%
\AddToShipoutPictureBG*{%
\AtPageUpperMyright{#1}
}
}

\begin{document}
\title{\vspace*{1cm} Logifold: A Geometrical Foundation of Ensemble Machine Learning\\
}

\author{
\IEEEauthorblockN{Inkee Jung}
\IEEEauthorblockA{\textit{Department of Mathematics and Statistics} \\
	\textit{Boston University}\\
	Boston, MA, USA \\
	inkeej@bu.edu}
	\and
\IEEEauthorblockN{Siu-Cheong Lau}
\IEEEauthorblockA{\textit{Department of Mathematics and Statistics} \\
	\textit{Boston University}\\
	Boston, MA, USA \\
	scllouis@bu.edu}

}
\maketitle
\conf{\textit{  Proc. of the International Conference on Electrical, Computer, Communications and Mechatronics Engineering (ICECCME 2024) \\ 
4-6 November 2024, Male, Maldives}}
\begin{abstract}
	We present a local-to-global and measure-theoretical approach to understanding datasets. The core idea is to formulate a logifold structure and to interpret network models with restricted domains as local charts of datasets.  In particular, this provides a mathematical foundation for ensemble machine learning. Our experiments demonstrate that logifolds can be implemented to identify fuzzy domains and improve accuracy compared to taking average of model outputs. Additionally, we provide a theoretical example of a logifold, highlighting the importance of restricting to domains of classifiers in an ensemble.
	
	
	\end{abstract}

\section{Introduction}
The concept of a manifold has been broadly used in data science for interpolating data points (for instance, \cite{b4} gives an excellent overview of the topic).
Recently, the study of dataset using topological methods develops into an interesting research area, see for instance \cite{b5, b6, b7}.
In most applications, manifolds are understood as higher dimensional analogs of surfaces in the Euclidean space $\mathbb{R}^3$.

On the other hand, a crucial aspect of a manifold is the local-to-global perspective to study spaces, which is often overlooked in applications to data science.  In \cite{b1} and this paper, we would like to formulate a local-to-global approach to study datasets.

Manifolds can be expressed as zero loci of smooth functions, which are well approximated by polynomials. In contrast,
datasets are like `point clouds' and not locally Euclidean.  Thus, we propose to model a dataset by a measure space.  Moreover, inspired by the huge success of neural networks, we take the graphs of linear logical functions as local models.

Linear logical functions are defined via graphs and linear inequalities, whose targets are finite sets.  We will restrict their domains to measurable subsets of $\mathbb{R}^n$.  Functions obtained by artificial neural networks belong to this class.
Linear logical functions are universal, in the sense that they can approximate any measurable functions with a finite target set. Consequently, constructing an atlas from the graphs of logical functions leads to an analog of a topological manifold in this setting, which we call to be a \emph{logifold}.

Fuzziness is another important aspect in our formulation.  Logical functions used in machine learning exhibit the characteristics of fuzzy logic, with values in the range $[0,1]$ rather than $\{0,1\}$. Thus, the graph of a logical function exhibit fuzziness, leading to the notion of a fuzzy logifold.  

Logifold provides a mathematical foundation for ensemble machine learning. Ensemble machine learning takes a weighted average of several models.  It has shown impressive results in classification problems (\cite{b8,b9}).  Moreover, it is shown to reduce bias and variance for clearer decision boundaries (\cite{b10,b11,b12,b13}). 

Our logifold formulation emphasizes the importance of restricting to the domain of each model when we take average.  Otherwise wrong predictions of  a model outside its domain can seriously harm the average accuracy.  We will make a theoretical example of a logifold to show the limitation of averaging over models.

In practice, the certainty scores given by the softmax function provide some information about the domain of a model.  This will be the main ingredient of implementing a logifold in practice.  
However, certainty score only contains partial information about the domain.  For instance, under adversarial attacks \cite{b2}, models have extremely low accuracy at input points that it has high certainty scores.  Thus, whenever possible, we shall make use of other means to restrict the domains of models.
Some important ones are the target classes of a model, the background context assumed by the model, and the statistics (allowed range, expected mean and variation) of input values.  For instance, in our implementation of a logifold, we allow models to have different types of target classes.   This also helps to increase the diversity and specificity of models.  

We carry out experiments to show the advantages of the logifold formulation.  Applying a fuzzy logifold to the CIFAR10(\cite{b14}) classification problem shows that it could achieve better accuracy improvements compared to weighted average or majority voting. In another experiment, we form a dataset with 30 target classes, which is the union of three, namely MNIST(\cite{b15}), Fashion MNIST(\cite{b16}), and CIFAR10.  We apply a logifold formulation that consists of models with different type of targets and show that this achieves higher accuracy than taking average of models with all the 30 targets.

For readers who focus on applications, they may want to first go to Section \MakeUppercase{\romannumeral 5} on experiments, and Section \MakeUppercase{\romannumeral 3} on comparison with ensemble learning and why it is important to restrict to domains of local models. Section \MakeUppercase{\romannumeral 2} and \MakeUppercase{\romannumeral 4} are devoted to a quick review of the mathematical definitions required for logifolds and fuzzy logifolds. We hope this work also serves as a motivation and bridge between data applications and various branches of mathematics.

\section{Linear Logical function}\label{sec: linear logical function}
In this section, we define the notion of a \emph{linear logical function} using a directed graph and affine maps associated with each node.  It formulates network models as a class of functions.

Let $D$ be a subset of $\mathbb{R}^n$ and $T=\{t_1, \ldots, t_k\}$ a finite set.
Given an affine linear map $l : \mathbb{R}^n \to \mathbb{R}$, the inequalities $l \geq 0$ and $l < 0$ define chambers, and we take their intersections with $D$.
Since intersections and unions can be expressed using Boolean algebra, we can interpret membership in a chamber using propositional logic based on inequalities of affine maps, for instance $x \in \mathbb{R}^{2}_{(+,+)} \leftrightarrow (x_1 \geq 0)\wedge (x_2 \geq 0)$.
We now define what a \emph{linear logical graph} is.
\begin{definition}[Linear Logical Graph]\label{def: linear logical graph}
	Let $G$ be a finite directed acyclic graph with vertex set $V$, having only one source vertex and a finite number of target vertices $t_1,\ldots,t_k$.  For each vertex $v$ with more than one outgoing arrows, we fix an affine linear functions whose chambers in $D$ are one-to-one corresponding to the outgoing arrows of $v$.  Let $L$ be the collection of such chosen functions.  
	The pair $(G,L)$ is called a linear logical graph.
\end{definition}

Given a linear logical graph, we can construct a linear logical function from $D$ to $T$.
\begin{definition}[Linear Logical Function]\label{def: linear logical function}
	For a given $x \in D$, we start from the source vertex of $G$, and follow a path according to the following rules. At a vertex $v$,
	\begin{itemize}
		\item if there is only one outgoing arrow, we follow that arrow.
		\item if there are multiple outgoing arrows, we consider the chambers defined by affine maps $l_v$. We select the outgoing arrow corresponding to the chamber that $x$ lies in.
	\end{itemize}
	Since $G$ has no cycle and finite, the path is finite and will stop at a target vertex, which is associated to an element $t \in T$. 
	This defines a function $f_{G,L}(x) :=t$, which is called a linear logical function determined by $(G,L)$.
\end{definition}

In classification problems, $T$ is the set of labels for target classes, and the dataset is a subset of $D \times T$, representing a graph of a function from $D$ to $T$.
For a feed-forward network model where the activation function at each hidden layer is the ReLU function, and the last layer uses the index-max function, the model takes the form 
\[
f(x) = \sigma \circ L_N \circ r_{N-1} \circ L_{N-1} \circ \cdots \circ r_1 \circ L_1(x)
\]
where $L_i : R^{n_{i-1}} \to R^{n_i}$ are affine linear functions with $n_0 = n$, $r_i$ are the entrywise ReLU functions, and $\sigma$ is the index-max function.

\begin{proposition}
	There exists a linear logical graph $(G,L)$ whose function $f_{(G,L)}$ gives the above $f$ composed of affine linear functions, the ReLu function, and the index-max function.
\end{proposition}

\begin{proof}
The graph is constructed as follows.
The source vertex $v_0$ is equipped with the affine linear function $L_1$, from which $N$ outgoing arrows originate, corresponding to the $N$ chambers created by $L_1$.
These chambers match the possible values of the ReLu function $r_1$, forming a finite subset of $\{0,+\}^{n_1}$. 
Next, we take the composition function $L_2 \circ r_1 \circ L_1$ restricted to each chamber, each of which is affine linear, and assign this function to the vertex corresponding to the chamber.
Again, we create outgoing arrows corresponding to the chambers in $\mathbb{R}^n$ defined by this linear function.
Inductively we proceed through the layer of vertices correspoding to the chambers of $L_N \circ r_{N-1} \circ L_{N-1} \circ \cdots \circ r_1 \circ L_1$.
Let $L_N = (l_1, \ldots, l_{n_N})$, and define $\tilde{L}_N := \left(l_i - l_j : i<j\right)$.
At each vertex, the function $\tilde{L}_N \circ r_{N-1} \circ L_{N-1} \circ \cdots \circ r_1 \circ L_1$ restricted to the corresponding chamber is linear on $\mathbb{R}^n$, and we equip this function to the vertex and make outgoing arrows corresponding to the chambers of the function. We assign this function to the vertex and create outgoing arrows corresponding to the chambers of the function.
Within each chamber, the index $i$ that maximizes $l_i \circ r_{N-1} \circ L_{N-1} \circ \cdots \circ r_1 \circ L_1$ is determined, and we make an outgoing arrow from the corresponding vertex to the target vertex $t_i \in T$.
\end{proof}

Below figure \ref{fig} shows an example of such a  linear logical graph, which produce a linear logical function of the form $\sigma \circ L_2 \circ r_{1} \circ L_{1} : \mathbb{R}^2 \to \{t_1, t_2, t_3\}$.
This corresponds to a three-layer feed-forward network model ($N=3$) with $L_1 = (l_{11},l_{12}) : \mathbb{R}^2 \to \mathbb{R}^4$ and $L_2=(l_{21},l_{22},l_{23}) : \mathbb{R}^4 \to \mathbb{R}^3$.
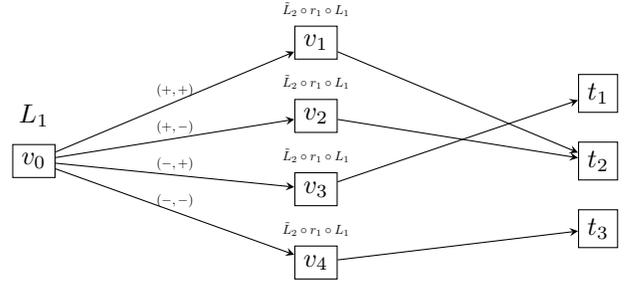
\tikzset{%
	every neuron/.style={
		draw
	},
	neuron missing/.style={
		draw=none, 
		scale=4,
		text height=0.333cm,
		execute at begin node=\color{black}$\vdots$
	},
}

\begin{tikzpicture}[x=1.5cm, y=1.5cm, >=stealth]
	
	\foreach \m/\l [count=\y] in {1}
	\node [every neuron/.try, neuron \m/.try] (input-\m) at (0,0.3) {$v_0$};
	
	\foreach \m [count=\y] in {1,2,3,4}
	\node [every neuron/.try, neuron \m/.try ] (hidden-\m) at (2.5,2-\y*0.65) {$v_\y$};
	\foreach \m [count=\y] in {1,2,3,4}
	\node [scale=0.5] at (2.5,2.3-\y*0.65){$\tilde{L}_2 \circ r_1 \circ L_1$};

	\foreach \m [count=\y] in {1,2,3}
	\node [every neuron/.try, neuron \m/.try ] (output-\m) at (5,1.5-\y*0.6) {$t_\y$};
%
%
	
	
	
	\draw [->] (input-1) -- (hidden-1)
	node [above, midway,scale=0.5] {$(+,+)$};
	\draw [->] (input-1) -- (hidden-2)
	node [above,midway,scale=0.5] {$(+,-)$};
	\draw [->] (input-1) -- (hidden-3)
	node [above, midway,scale=0.5] {$(-,+)$};
	\draw [->] (input-1) -- (hidden-4)
	node [above, midway,scale=0.5] {$(-,-)$};
	\node at (0,0.7){$L_1$};
	
	\draw [->] (hidden-1) -- (output-2);
	\draw [->] (hidden-2) -- (output-2);
	\draw [->] (hidden-3) -- (output-1);
	\draw [->] (hidden-4) -- (output-3);
	
\end{tikzpicture}
\begin{figure}[htbp]
	\captionsetup{singlelinecheck=false, font=myfont}
	\caption{Example of a linear logical graph.}
	\label{fig}
\end{figure}

\subsection{Universality of linear logical functions and definition of a logifold.}
We introduce the universal approximation theorem for linear logical functions. 

\begin{theorem}[Universal Approximation Theorem for measurable functions(\cite{b1})]\label{thm:UAT}
Equip $\mathbb{R}^n$ with the Lebesgue measure $\mu$.  Let $D \subset \mathbb{R}^n$ be a measurable subset with $\mu(D)<\infty$, and $T$ a finite set.
For any measurable function $f : D \to T$ and $\epsilon>0$, there exists a linear logical function $L:D \to T$ and a measurable set $E \subset D$ with $\mu(E)<\epsilon$ such that $L|_{D-E} \equiv f|_{D-E}$. 

Moreover, there exists a countable family $\mathcal{L}$ of linear logical functions $L_i : D_i \to T$ where $D_i \subset D$ and $L_i \equiv f|_{D_i}$, such that $D \setminus \bigcup_{i=0}^\infty D_i$ is measure-zero set.
\end{theorem}

We note the following main differences from usual approximations by polynomials or Fourier series.  First, linear logical functions are discontinuous.  Second, the target set is finite here.  In particular, the approximation $L$ exactly equals the target function $f$ in a large portion $D-E$ of $D$.

The second part of the theorem tells us that up to a measure-zero set, the graph of $f$ is covered by the graphs of $L_i$.  This makes an analog of a manifold and motivates the following.
\begin{definition}[Linear Logifold]
	A linear logifold is a pair $\left(X, \mathcal{U}\right)$, where $X$ is a topological space equipped with a $\sigma-$algebra and a measure $\mu$, $\mathcal{U}$ is a collection of pairs $\left(U_i, \phi_i\right)$ where $U_i$ are subsets of $X$ such that $\mu(U_i) >0$ and $\mu(X-\bigcup_{i}U_i)=0$; $\phi_j$ are measure-preserving homeomorphisms between $U_i$ and the graphs of linear logical functions $f_i : D_i \to T_i$, where $D_i \subset \mathbb{R}^{n_i}$ are measurable subsets and $T_i$ are discrete sets.
\end{definition}

To make a simple example, let $S \subset (-\infty,0]$ be any measurable subset, $X \subset \mathbb{R} \times \{1,2,3\}$ be the graph of the function $f$ defined by $f(x)=1$ if $x \in S \cap (-\infty,0]$, $f(x)=2$ if $x>0$, and $f(x)=3$ if $x \in (-\infty,0] \backslash S$. Let $T_1=\{1,2\}$ and $T_2=\{2,3\}$. For $i=1,2$, $f_i := f|_{f^{-1}(T_i)}$ are the restrictions of $f$ on the inverse sets $D_i := f^{-1}(T_i) \subset \mathbb{R}$; $U_i \subset D_i \times T_i$ are the graphs of $f_i$; $\phi_i$ is simply identity. Note that $f_i$ can be expressed as the step function on $\mathbb{R}$ restricted on $D_i$ and hence is a linear logical function, even though $f$ itself is not.

\begin{remark}
\begin{enumerate}
	\item In the above definition, topological structure is necessary to distinguish graphs of different linear logical functions.  
	\item In place of openness condition on local charts for topological manifolds, we make a measure-theoretical requirement $\mu(U_i)>0$, which is more relevant to datasets.
	\item The graphs of $f_i$ are equipped with measures induced from $D_i$.
\end{enumerate}

\end{remark}

\section{Ensemble machine learning and a theoretical example}
Ensemble methods have achieved significant success in many classification problems(\cite{b8,b9}).
These methods improve model performance not only in terms of accuracy but also in computational efficiency.
Essentially, ensemble methods combine multiple classifiers to reduce bias and variability, or to create clearer decision boundaries(\cite{b10}-\cite{b13}).

In ensemble theory, various techniques are employed to enhance performance.
One approach involves training models on evenly divided training datasets, or re-training models on misclassified datasets to form an ensemble (\cite{b10, b12}).
Another technique adaptively adjusts the weights during the final combining step of the ensemble, as models are transferred and the final combining layer is adaptively trained (\cite{b18, b8}).
Additionally, some studies (\cite{b19, b20}) explore finding the best representative point in a dataset using geometric methods.

Since models in an ensemble learning have discrepancies and may disagree with each other, increasing the number of models in an ensemble often degrade the overall accuracy, even if computational costs are ignored. An optimal selection of the ensemble components is necessary (\cite{b19, b20}).  

The above logifold formulation provides a mathematical foundation for ensemble learning.  It emphasizes on keeping track of the domain of each local chart.  Theorem \ref{thm:UAT} ensures that such a local-to-global formulation allows a theoretical zero loss in accuracy (which means the measure of wrongly predicted subset is zero).

Consider a classifier as an approximation of a target function. In the context of linear logical functions, the \emph{model size} can be defined as the number of chambers in its domain. 

Let us fix $N>0$ to be the model size.  Below, we construct a simple example of a logifold, such that an ensemble of any number of models (which are linear logical functions) with size $\leq N$ has accuracy below a fixed bound less than $1$.  This shows the importance of keeping to the domains of local models.

For $E\subset\mathbb{R}$, \emph{the characteristic function} of $E$ is 
\[I_{E}(x) := \begin{cases}
	1 \quad \text{if } x \in E \\
	0 \quad \text{otherwise.}
\end{cases}\]
Let $f : (0,1] \to \{0,1\}$ be a function defined as 
\[f(x) = \sum_{n=0}^{\infty}\left( \frac{(-1)^n + 1}{2} \right)I_{E_n}(x)\]
where $E_n = (2^{-n-1},2^{-n}]$.
The graph of $f$ represents a dataset in $(0,1]\times\{0,1\}$ with countably many `jumps' or `discontinuities' near at $x=0$.

Any linear logical function $f : (0,1] \to \{0,1\}$ must take the form \[\sum_{j=0}^{N}\left( \frac{(-1)^n + 1}{2} \right)I_{A_{j}}(x) \] where $A_j$ are sub-intervals in $(0,1]$ satisfying $\bigcup_{j=0}^{N}A_j = (0,1]$ and $A_j \cap A_k =\emptyset$ for $j\neq k$.

Consider \emph{a family of linear logical functions} $\mathcal{F} = \{f_i : i=1,\ldots,K\}$ be a finite family of functions
\[f_i(x) : = \sum_{j=1}^{N}\left( \frac{(-1)^n + 1}{2} \right)I_{A_{i,j}}(x) \]
where $A_{i,j} = (\alpha_{i,j+1},\alpha_{i,j}], 0=\alpha_{i,N+1} <  \alpha_{i,N} \leq \cdots \leq \alpha_{i,1} \leq \alpha_{i,0} = 1$, and 
\[g_\mathcal{F}(x) := \sigma\circ\left(\frac{1}{K}\sum_{i=1}^K f_i(x)\right)\] where $\sigma(x) = I_{\{x\geq \frac{1}{2}\}}(x)$.
Each $f_i$ stands for a trained \emph{model} (or \emph{classifier}) to the dataset $\Gamma(f)$. Note that each $f_i$ is linear logical function.
Then $g_\mathcal{F}$ serves as the ensemble of $\mathcal{F}$.  The \emph{consistency} of the family $\mathcal{F}$ at $x\in D$ is defined by $\frac{\max(M_0(x), M_1(x))} {K} $.

The upper bound $N$ for the model size represents the number of \emph{decision boundaries}, or discontinuities of $f_i$.  In other words, each classifier has at most $(N+1)$ chambers in its domain. 

Typically, the combiner of an ensemble uses a weighted average.
Alternatively, the effect of weights can be realized by adding more instances of the same logical function in the family.
Let $M_t(x) $ denote the number of functions $f_i \in \mathcal{F}$ such that $f_i(x) = t$ for $x \in (0,1]$ and $t \in \{0,1\}$.

It is typically assumed that each classifier in the ensemble is at least a weak learner, with prediction accuracy as well as random guessing, approximated by $\frac{1}{T}$ where $T$ is the number of targets \cite{b10}.
Therefore, we set $\frac{T+1}{T^2}$ as the least consensus of consisting models of an ensemble. 
The following theoretical example demonstrates that ensemble machine learning fail to fully describe a dataset.
Here, the consistence number $\frac{3}{4}$ is derived from $\frac{T+1}{T^2}$ where $T=2$.
\begin{theorem}

There exists $\delta > 0$ such that 
\[\mu\left(\left\{x \in (0,1]: f(x) = g_\mathcal{F}(x) \right\}\right) \leq 1-\delta\] for any finite family $\mathcal{F}$ of logical functions of size at most $N$ if the consistency of $\mathcal{F}$ is greater than $ \frac{3}{4}$ for all $x \in (0,1]$, where $\mu$ denotes the Lebesgue measure.
\end{theorem}

\begin{proof}
	For any $x \in \mathbb{R}$, define $\lfloor x \rfloor:= \sup \{n \in \mathbb{Z} : n \leq x\}$.
	We denote $K$ be the cardinality of $\mathcal{F}$.
	Let $\mathcal{A}_{\mathcal{F}}$ denotes the set of discontinuities of $g_\mathcal{F}$ and $M_\mathcal{F} := \left|\mathcal{A}_{\mathcal{F}}\right|$.
	Then we can enumerate $\mathcal{A}_{\mathcal{F}}$ as $\{\alpha_1 , \ldots, \alpha_{M_{\mathcal{F}}}\}$ where $1>\alpha_1 > \alpha_2 \cdots > \alpha_{M_{\mathcal{F}}}>0$.
	Define $\alpha_0 = 1$ and $\alpha_{M_{\mathcal{F}}+1}=0$.
	Let $U_i$ be the number of functions $f \in \mathcal{F}$ such that $f(x) = 1$ on the interval $(\alpha_{i+1}, \alpha_i]$ for $i=0, \ldots, M_{\mathcal{F}}$.
	Define $\Delta_i := \left| U_{i+1} - U_i \right|$ as the difference between $U_i$ and $U_{i+1}$ for $i=0, \ldots, M_{\mathcal{F}}$.
	
	By consistency of $\mathcal{F}$, either $U_i>\frac{3}{4}\left\lfloor \frac{K}{4} \right\rfloor$ or $U_i\leq \frac{1}{4}\left\lfloor \frac{K}{4} \right\rfloor$ for all $i$, and $\Delta_i \geq 2\left\lfloor \frac{K}{4} \right\rfloor$.
	Since each function $f \in \mathcal{F}$ can have at most $N$ many discontinuities, we have \[\sum_{i=0}^{M_{\mathcal{F}}} \Delta_i \leq KN,\]  and therefore \[  2\left\lfloor \frac{K}{4} \right\rfloor \left(M_{\mathcal{F}}+1\right) \leq KN\], which implies that $g_\mathcal{F}$ have at most $\frac{KN}{2\left\lfloor \frac{K}{4} \right\rfloor} -1 $ number of discontinuities.
	Since $2N-1 \leq L < 3N-1$ where $L=\frac{KN}{2\left\lfloor \frac{K}{4} \right\rfloor} -1$, we have $\mu( \{ f = g_\mathcal{F}\}) < 1 -3\cdot{2^{-3N}}$.
	
\end{proof}

\begin{remark}
	In the proof, $L$ can be expressed as $2N-1 + \mathcal{O}(K^{-1})$, that means to increase the size of $\mathcal{F}_K$ is to decrease the upper bound to $1-3\cdot 2^{-2N}$.
	Additionally, the function $g_K$ has $\sigma$ at the last `layer', analogous to the index-max function.
\end{remark}

\section{Fuzzy linear logical function and Fuzzy logifold}
Deep neural network model provides a function $f : D \to T$ for classification problems, where the last layer is typically given by the SoftMax function rather than the index-max function, and activation functions are sigmoidal functions valued in $[0,1]$ instead of $\{0,1\}$.
A classifier describes the dataset with fuzziness in its prediction, leading to the formulation of \emph{fuzzy linear logical function and graph}.
\begin{definition}[Fuzzy linear logical function and graph \cite{b1}]
	$G$ be a graph as in definition \ref{def: linear logical graph}.
	Each vertex $v$ of $G$ is equipped with a product of standard simplices
	\[P_v = \prod_{k=1}^{m_v} S^{d_{v,k}}\] where \[S^{d_{v,k}} = \left\{(y_0,\ldots,y_{d_{v,k}}) \in \mathbb{R}_{\geq 0}^{d_{v,k}+1}: \sum_{i=0}^{d_{v,k}} y_i = 1\right\}\]
	for some integers $m_v>0$, $d_{v,k}\geq 0$.  $P_v$ is called the internal state space of the vertex $v$.  Let $D$ be a subset of the internal state space of the source vertex of $G$.  Each vertex $v$ that has more than one outgoing arrows is equipped with an affine linear function 
	\[ l_v: \prod_{k=1}^{m_v} \mathbb{R}^{d_{v,k}} \to \mathbb{R}^j \]
	for some $j > 0$, whose chambers in the product simplex $P_v$ are one-to-one corresponding to the outgoing arrows of $v$.  (In above, $\mathbb{R}^{d_{v,k}}$ is identified with the affine subspace $\left\{\sum_{i=0}^{d_{v,k}} y_i = 1\right\}$ that contains $S^{d_{v,k}}$.)
	Let $L$ denote the collection of these linear functions.
	Moreover, each arrow $a$ is equipped with a continuous function 
	\[ p_a: P_{s(a)} \to P_{t(a)} \]
	where $s(a),t(a)$ denote the source and target vertices respectively.  
	
	We call $(G,L,P,p)$ a fuzzy linear logical graph.  $(G,L,P,p)$ determines a function 
	\[f_{(G,L,P,p)}: D \to P^{\textrm{out}} := \coprod_{l=1}^K P_{t_l}\]  
	as follows.
	Given $x \in D$, the collection $L$ of linear functions over vertices of $G$ evaluated at the image of $x$ under the arrow maps $p_a$ determines a path from the source vertex to one of the target vertices $t_l$.  By composing the corresponding arrow maps $p_a$ on the internal state spaces along the path and evaluating at $x$, we obtain a value $f_{(G,L,P,p)}(x) \in P_{t_l}$.  The resulting function $f_{(G,L,P,p)}$ is called a fuzzy linear logical function.
\end{definition}

	In a feed-forward network model $f$, where the activation function at each hidden layer is the ReLU function $s$ and that at the last layer is SoftMax, we can construct a linear logical graph $(G,L)$ as in section \ref{sec: linear logical function}. The last two layers of vertices in $G$ are replaced by a single vertex $t$, and the affine maps of $L$ at vertices adjacent to the last target $t$ and $t$ itself are excluded, since these vertices do not have more than one outgoing arrow.
	The internal state space $P_v$ at each vertex $v$ is $\left(S^1\right)^{n}$, except for the target vertex $t$, and $P_t := S^d$ where $d$ is the number of target classes. 
	The map $p_a$ is defined to be the identity function for all arrows $a$, except for those targeting $t$.
	On each adjacent vertex of the target $t$, $p_a$ for the corresponding arrow $a$ is defined to be $\text{SoftMax} \circ l$ where $l$ is the restricted affine map on the corresponding chamber in $\mathbb{R}^n$. In this construction, we get $f = f_{(G,L,P,p)}$. 
	
We call \emph{the corners} of the set $P_v$ be state vertices, which takes the form $e_I = (e_{i_1},\ldots,e_{i_{m_v}}) \in P_v$ for a multi-index $I=(i_1,\ldots,i_{m_v})$, where $\{e_0,\ldots,e_{d_{v,k}}\} \subset \mathbb{R}^{d_{v,k}+1}$ is the standard basis.
\begin{definition}[Fuzzy linear logifold \cite{b1}]
	A fuzzy linear logifold is a tuple $(X,\mathcal{P},\mathcal{U})$, where 
	\begin{enumerate}
		\item $X$ is a topological space equipped with a measure $\mu$; 
		\item $\mathcal{P}: X \to [0,1]$ is a continuous measurable function;
		\item $\mathcal{U}$ is a collection of tuples $(\rho_i,\phi_i, f_i)$, where $\rho_i$ are measurable functions $\rho_i: X \to [0,1]$ with $\sum_i \rho_i \leq 1_X$ that describe fuzzy subsets of $X$ 
		whose supports are denoted by $U_i = \{x\in X: \rho_i(x)>0\} \subset X$; 
		$$\textstyle \phi_i: U_i \to D_i \times T_i$$ 
		are measure-preserving homeomorphisms where $D_i \subset \mathbb{R}^{n_i}$ are (Lebesgue) measurable subsets in certain dimension $p_i$; $f_i$ are fuzzy linear logical functions on $D_i$ whose target sets are $T_i$;
		\item The induced fuzzy graphs $\mathcal{F}_i: D_i \times T_i \to [0,1]$ of $f_i$ satisfy
		$
		\mathcal{P} = \sum_i \rho_i\cdot \phi_i^*(\mathcal{F}_i). 
		$
	\end{enumerate} 
\end{definition}

\section{Experiments}
We conducted experiments on the logifold formulation using commonly used dataset such as MNIST \cite{b15}, Fashion MNIST \cite{b16}, and CIFAR10 \cite{b14} dataset. Our experiments aim at comparing our refined voting system based on fuzzy domains with taking average in usual ensemble learning \emph{for the same collection of models}. Our experiments were in small scale and did not take a collection of huge models for comparison with state-of-the-arts models. We wish to carry out experiments in a larger scale and for a broader range of datasets in the future.

\subsection{Experiment Setting}
We trained the models using either a Simple CNN structure or ResNet \cite{b22} for $200$ epochs and with a batch size of $32$.
The learning rate had been set to $10^{-3}$ initially, and then decreased to $0.5\times 10^{-12}$, with the data augmentation. We used the ADAM \cite{b23} optimizer with $\eta = 10^{-3}$. 
The logifold formulation was implemented, and its pseudo-code can be found in \cite{b1}.

For each classifier $M$, we stratify our domain, which is called \emph{fuzzy domain} of $M$(\cite{b1}), according to the certainty of classifier as provided by $M(x)$ to restrict domain of the classifier.
In experiments, we preset our certainty thresholds as $\mathcal{A} =\{0,\sigma(0),\sigma(0.5),\ldots,\sigma(10)\}$ where $\sigma(x)= \left(\exp{(1+e^x)}\right)^{-1}$.
When evaluating a logifold formulation, we use the fuzzy domain of each classifier to combine their predictions for a given instance $x$.
Specifically, a model provides its prediction only when its certainty exceeds a predefined threshold.

With four models (ResNet20 and ResNet56) on the CIFAR-10 testing dataset, using simple averaging, the ensemble achieved an accuracy of 86.97\%. In our logifold formulation \cite{b1}, the accuracy increased to 93.23\%.

\subsection{Experiment 1: Ensemble of six Simple CNN and one ResNet20 on CIFAR10}
The purpose of this experiment is to test whether our logifold program can defend against bad outputs from models with low accuracy. 
We used the test set of $10,000$ images of CIFAR10 for evaluation. 
Six of the models were trained with a Simple CNN structure, and the other was with the version 1 of ResNet20 strucuture.
The Simple CNN models had relatively low accuracy, 56.45\% in average, while the ResNet model achieved 85.96\% accuracy.
Combining them, we obtained the following table \ref{tab1}. The first column of the table shows the certainty thresholds for the fuzzy domain of each model, and the third column shows the number of data points for which the logifold, as the combiner of classifiers, had greater certainty than the certainty threshold. The second column is the accuracy in the restricted domain, not in the entire test dataset.
\begin{table}[htbp]
	\captionsetup{font=myfont}
	\caption{Result of Experiment 1}
	\begin{tabular}{c||c|c|c}
		\centering
		Certainty  & Accuracy  & Accuracy  & The number of \\
		threshold& with refined voting& in certain part& certain data\\
		\hline
		0 & 0.6158  & 0.6158&10000\\
		0.8808 &0.7821   & 0.7821 & 10000 \\
		0.9526 &0.8185   &  0.8946 & 8653 \\
		0.9997 &0.6544  & 0.9984 &2495\\
		\hline
	\end{tabular}\\ \label{tab1}\\
	Simple average: 62.55\%\\ Majority voting: 58.72\%\\
	Our logifold: 84.86\%
\end{table}
As expected, both weighted(or simple) average and majority voting performed poorly due to the predominance of the six Simple CNN models.
In contrast, our refined voting system using fuzzy domains yields about 81.85\% accuracy at the certainty threshold $0.9526$.  By also utilizing validation history \cite{b1} to choose the optimal certainty threshold for each testing data entry, we can achieve 84.86\% accuracy. This shows our logifold have suppressed most of the wrong outputs of the Simple CNN models.

\subsection{Experiment 2: union dataset of MNIST, Fashion MNIST and CIFAR10}
In this experiment, we concatenate the MNIST, Fashion MNIST and CIFAR10 datasets into a single dataset with 30 target classes (by resizing all images to $32\times 32\times 3$ input dimensions). We construct a logifold for
\[\mathbb{R}^{32} \times \mathbb{R}^{32}\times\mathbb{R}^3 \to T = \{m_i,f_i,c_i : i=0,\ldots,9\}\] where $m_i, f_i,$ and $c_i$ denote the target classes of MNIST, Fashion MNIST, and CIFAR10 respectively. The structure of each model is ResNet. The special feature of our logifold is that it consists of models with different target classes. In particular, each model only covers a part of the dataset by design. We show that such a logifold outperforms learning of an ensemble of models which have all the 30 target classes.

Our logifold consists of five types of models, $T, M, F, C,T^{\prime}$ to be explained below.
For the models trained on MNIST and Fashion MNIST, a random selection of 50,000 samples from the respective datasets was chosen because the CIFAR-10 dataset has 50,000 samples in the training dataset.
This selection was then split into training and validation sets with a ratio of $0.2$, resulting in 40,000 samples for training and 10,000 for validation.

First, we trained models $T$  with a concatenated dataset that consists of one-third of each type of data, totaling 39,999 training samples and 9,999 validation samples. Here, type $T$ model refers to a model classifying 30 classes which is the union of classes in MNIST, Fashion MNIST, and CIFAR10.
Its testing accuracy in 30,000 concatenated samples of testing dataset was 76.41\% in average.

Next, we derived a model of type $T^\prime$ using `specialization' method(\cite{b1}) from model of type $T$, which is a classifier trained on the 40,000/10,000 training/validation samples having only three targets $\{m_0,\ldots,m_9\},\{f_{0},\ldots,f_{9}\},\{c_{0},\ldots,c_{9}\}$.
Moreover, we trained models that have partial targets.
Each partial model was trained exclusively on one of the datasets.
When we trained models for MNIST type dataset, RGB channels are randomly given.
Let $M, F, C$ be the following three types:
\begin{enumerate}
	\item $M$ has specialized on MNIST dataset with classes  $\{m_i\}$.
	\item $F$ does on Fashion MNIST dataset with classes $\{f_i\}$.
	\item $C$ does on CIFAR10 dataset with classes $\{c_i\}$.
\end{enumerate}

The testing phase involved a concatenated dataset comprising 30,000 testing samples from MNIST, Fashion MNIST, and CIFAR10.
The logiold of models of type $\{ M, F, C,T^{\prime}\}$ yields an accuracy of 94.90\% at 0.9526 certainty threshold and 94.94\% based on validation history,  which is greater than 82.35\% the accuracy of ensemble of $T$ type models only. See the Table \ref{tab2} below for the detailed experimental result.  
\begin{table}[htbp]
	\captionsetup{ font=myfont}
	\caption{Results of Experiment 2}
	\begin{tabular}{c||c|c|c}
		\centering
		Certainty  & Accuracy  & Accuracy  & The number of \\
		threshold& with refined voting& in certain part& certain data\\
		\hline
		0 & 0.9483  & 0.9483&30000\\
		0.8808 &0.9486   & 0.9486 & 30000 \\
		0.9526 &0.9490   &  0.9508 & 29870 \\
		0.9997 &0.9486  & 0.9942 &19777\\
		\hline
	\end{tabular} \label{tab2}\\ \\
	Single model with full target of 30 classes: 76.41\% in average.\\
	Simple average of four models with 30 classes: 82.35\%. \\
	Our logifold : 94.94\%
\end{table}
\section{Conclusion}


The first experiment shows that even when models of poor performance predominantly occupy our logifold formulation, the main contribution comes from the most well-trained component, unlike usual ensemble methods. The second experiment demonstrates that the logifold formulation can flexibly combine models with different target types; they are local in the sense that their domains are strict subsets of $\mathbb{R}^n \times T$ where $T$ denotes the target set of classes. We utilize such flexibility to achieve higher accuracy (more than 10\%) than usual averaging method of ensemble learning.

In summary, our logifold formulation provides a geometric formulation of ensemble learning, and the refined voting method using fuzzy domains achieves significant improvement in accuracy in experiments.
%
%

%
%
%
%

\end{document}